\documentclass[letterpaper,10pt,conference]{ieeeconf}
\IEEEoverridecommandlockouts
\usepackage{multirow}
\usepackage{lipsum}
\usepackage{amsmath}
\usepackage{amssymb}
\usepackage[ruled,vlined]{algorithm2e}
\usepackage{graphicx}
\usepackage{url}
\usepackage{cite}

\newtheorem{proposition}{Proposition}
\graphicspath{{./Figures/}}

\title{\LARGE \bf Directional Pullback Regularization for Long-Horizon Rollout of Latent World Models in Autonomous Driving}

\author{
	\parbox{\textwidth}{%
		\centering
		Shatadal Mishra, Ahmadreza Moradipari, Nejib Ammar
	}%
	\thanks{Toyota Motor North America R\&D - InfoTech Labs,
		{\tt\small \{shatadal.mishra, ahmadreza.moradipari, nejib.ammar\}@toyota.com,}}%
}

\begin{document}
	
	\maketitle
	\thispagestyle{empty}
	\pagestyle{empty}
	
\begin{abstract}
Learned world models for autonomous driving operate in unconstrained latent spaces, yet the physical systems they represent are governed by strict non-holonomic constraints: ground vehicles cannot realize instantaneous pure lateral motion. We study long-horizon latent rollout modeling for autonomous driving and ask whether rollout quality can be improved by assigning a motion-aware geometry to latent transitions, rather than by redesigning the predictor architecture. Starting from an unconstrained latent rollout baseline, we introduce \emph{Directional Pullback Regularization} (DPR), a training-time method that defines a kinematics-aware anisotropic Riemannian metric in decoded state space and pulls it back through the decoder Jacobian to the latent manifold. The resulting latent-space metric assigns higher cost to latent directions that decode to lateral vehicle motion, while preserving the original encoder--dynamics--decoder backbone. This yields a clean three-way comparison among an \emph{Unconstrained Latent Rollout} model (ULR), an \emph{Euclidean Pullback Regularization} ablation (EPR), and the full directional method (DPR). We evaluate all models on open-loop ego-motion rollout over a common 10-second test horizon, using training rollout lengths of 1, 3, 5, and 10 steps. Across all four training regimes, DPR achieves lower 10-second final displacement error than both ULR and EPR. The strongest relative accuracy gain appears for models trained with 3-step rollout supervision, where DPR reduces the 10-second final displacement error by approximately \(44\%\) relative to ULR. The strongest overall control-relevant tradeoff is obtained for longer rollout-trained models, especially at 10 steps, where DPR simultaneously improves displacement error, lateral-motion consistency, and forward-motion preservation relative to both ULR and EPR.

\end{abstract}

\section{Introduction}
Ground vehicles obey non-holonomic kinematic constraints: instantaneous motion must align with the vehicle heading, so a car cannot move by pure sideways slip. Standard latent world models impose no corresponding geometric constraint on decoded state increments, creating a mismatch in which latent transitions can decode to kinematically inconsistent motion. Over long roll-out horizons, this mismatch can accumulate even when the loss of one-step prediction is small.

We ask whether this mismatch can be corrected without changing the predictor architecture. To that end, we introduce \emph{Directional Pullback Regularization} (DPR), which defines a kinematics-aware anisotropic Riemannian metric in decoded state space and pulls it back through the decoder Jacobian to latent space. In the spirit of prior work on decoder-induced Riemannian geometry and pullback metrics for latent-variable models \cite{tosi2014metrics,arvanitidis2018latent,shao2018riemannian,arvanitidis2021geometrically}, DPR penalizes latent directions that decode to lateral vehicle motion while preserving the original encoder--dynamics--decoder backbone, thereby acting as a training-time geometric bias for long-horizon rollout consistency.

Building on a mature lineage of latent world models
\cite{mishra2026network,worldmodels,planet,dreamer,recurrent_world_models,dreamerv3,tdmpc2},
the dominant trend in driving world models has been toward richer
architectures, larger generative stacks, and more specialized scene
representations. Autonomous driving is a particularly compelling setting
for world models because safe decision-making depends on anticipating
the consequences of ego motion in complex, interactive, and partially
observed scenes. Recent surveys emphasize that world models are becoming
a central tool for simulation, forecasting, and planning in driving
\cite{wm_survey_ad}. Representative examples include TrafficBots for
data-driven traffic simulation and motion prediction \cite{trafficbots},
GAIA-1 and GAIA-2 for controllable generative driving world models
\cite{gaia1,gaia2}, Copilot4D for tokenized discrete-diffusion world
modeling \cite{copilot4d}, BEVWorld for multimodal bird's-eye-view
latent simulation \cite{bevworld}, and Drive-OccWorld for
action-conditioned four-dimensional occupancy forecasting and planning
\cite{driveoccworld}. Collectively, these works show that learned latent
representations can support increasingly realistic and controllable
autonomous-driving world models. However, they rely on heavyweight
video-generation pipelines, specialized scene representations, or
auxiliary components that make it difficult to isolate a simple geometric
mechanism for improving ego-motion rollout. As a result, they do not
directly answer a narrower control-oriented question: can long-horizon
latent rollout be improved by shaping the local geometry of latent
transitions themselves?

This question is especially important because long-horizon latent rollout
remains fragile. A model may reduce one-step prediction loss while still
drifting toward motion patterns that are difficult to reconcile with the
local geometry of feasible vehicle motion \cite{fazliani2025enhancing}. In driving, this issue is
particularly acute because local motion is anisotropic: forward
translation, heading evolution, and speed change are natural directions
of motion, whereas pure lateral displacement is kinematically forbidden
by the non-holonomic constraint. Recent efforts have therefore begun to
inject stronger physical or kinematic structure into learned world
models. Examples include physically informed driving world models,
physically interpretable world models, and kinematics-aware latent world
models for autonomous driving \cite{drivephysica,piwm,kaware}. These
works reinforce an important point: unconstrained latent prediction is
often insufficient in safety-critical dynamical domains.

However, existing physically grounded approaches often introduce
structure through auxiliary modules, architectural changes, additional
supervision, or specialized representations \cite{drivephysica,piwm,kaware,fazliani2026shapebench}.
In contrast, this paper asks a more targeted question: can one improve a
latent rollout model by regularizing the decoded tangent geometry induced
by the latent transition itself, without redesigning the predictor? This question is attractive from a control perspective because the local rollout tangent directly determines how a latent model
moves through decoded state space and how error accumulates during
open-loop prediction, which is central to planning and model-
based control. If that tangent is geometrically misaligned with feasible motion, long-horizon rollout can degrade even
when one-step fitting appears acceptable.

The proposed framework yields a clean three-way decomposition. The first
model is \emph{Unconstrained Latent Rollout} (ULR), which serves as the
baseline trained only with prediction loss. The second is \emph{Euclidean
Pullback Regularization} (EPR), obtained by removing directional
preference from the pullback metric; this ablation penalizes overall
decoded tangent magnitude but does not distinguish between feasible and
lateral motion. The third is the full \emph{Directional Pullback
Regularization} (DPR) model, which adds explicit directional cost to the
lateral decoded component. This decomposition separates three distinct
questions: whether pullback regularization improves over unconstrained
training, whether any gain can be explained by generic smoothing alone,
and whether an explicitly directional geometric penalty is necessary.

Our central hypothesis is that a pullback metric aligned with
vehicle-motion geometry can improve long-horizon rollout quality more
effectively than either unconstrained training or Euclidean pullback
regularization alone. To test this hypothesis, we evaluate ULR, EPR,
and DPR using long-horizon open-loop ego-motion rollout metrics that are
directly relevant to control-oriented world modeling, including
displacement error, lateral-motion error, curvature consistency, and
preservation of forward motion.

The contributions of this work are threefold:
\begin{itemize}
    \item To the best of our knowledge, this is the first work to study
    directional pullback regularization for long-horizon latent ego-motion
    rollout in autonomous driving, using an explicit projector-based motion
    metric that penalizes decoded lateral motion while preserving the
    original encoder--dynamics--decoder backbone.
    \item We provide a clean experimental decomposition among the
    unconstrained baseline (\emph{Unconstrained Latent Rollout}, ULR),
    the Euclidean pullback ablation (\emph{Euclidean Pullback
    Regularization}, EPR), and the full directional method
    (\emph{Directional Pullback Regularization}, DPR).
    \item We characterize the interaction between pullback strength
    $\beta$ and training rollout length $U$, showing that longer self-fed
    rollout training allows DPR to simultaneously improve displacement
    accuracy, lateral consistency, and preservation of forward motion.
\end{itemize}

Taken together, these contributions provide a concise and
control-relevant perspective on how latent world models can be shaped by
the geometry of feasible motion rather than by prediction loss alone.
\begin{figure*}[t]
  \centering
  \includegraphics[width=\textwidth]{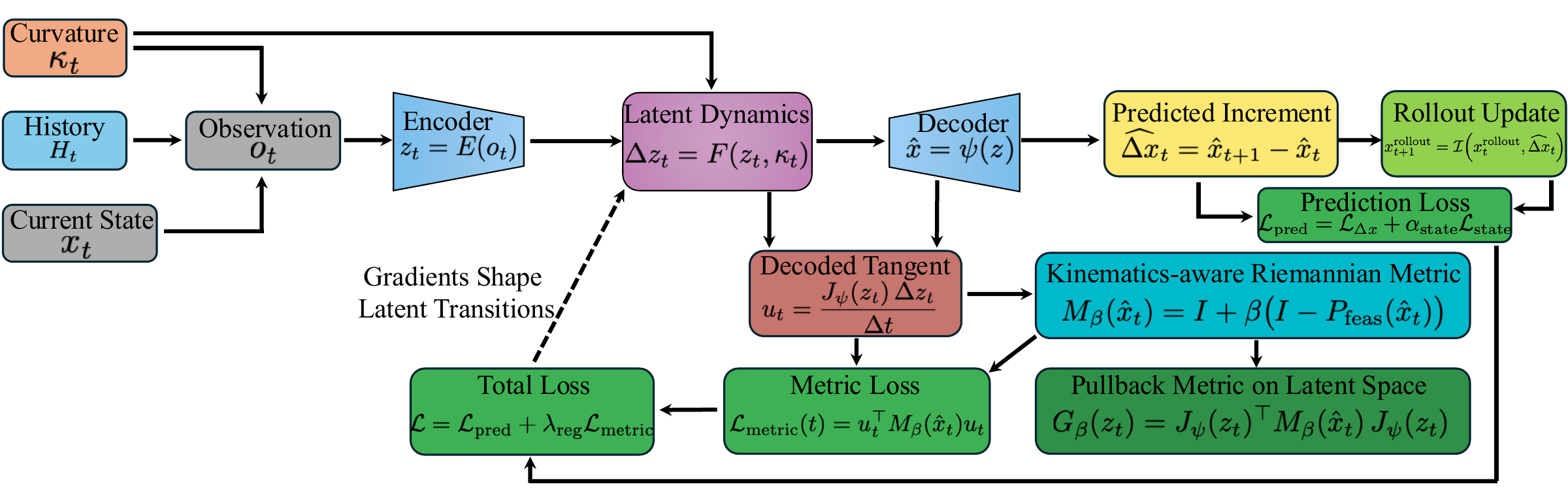}
  \caption{Method overview of the proposed directional pullback regularization (DPR) framework. The upper row shows the baseline open-loop rollout pipeline: the motion history, current state, and road curvature are assembled into the observation \(o_t\), encoded into the latent state \(z_t\), propagated by the latent dynamics model to produce \(\Delta z_t\), decoded back to state space, converted to the predicted increment \(\widehat{\Delta x}_t\), and integrated to obtain the next rollout state. The lower row shows the regularization path. The latent update and decoder Jacobian define the decoded tangent \(u_t = J_{\psi}(z_t)\Delta z_t/\Delta t\), which is evaluated under the kinematics-aware Riemannian metric \(M_\beta(\hat{x}_t)\). Pulling this metric back through the decoder Jacobian yields the latent-space metric \(G_\beta(z_t)\), and the resulting metric loss is combined with the prediction loss to form the total training objective. The dashed arrow indicates the backpropagated training signal through which the pullback regularizer shapes latent transitions.}
  \label{fig:methodology}
\end{figure*}
\section{Problem Setup and Method}
Figure~\ref{fig:methodology} summarizes the proposed training pipeline. The upper row shows the baseline latent rollout model, which predicts state increments in open loop from motion history, current state, and road curvature. The lower row shows the additional geometric regularization path introduced by the proposed method: the latent update induces a decoded tangent through the decoder Jacobian, this tangent is evaluated under a kinematics-aware decoded-space metric, and the resulting pullback metric loss is combined with the standard prediction loss during training.

\subsection{Problem Setup}

We study latent rollout modeling for ego-motion prediction in autonomous driving. To avoid overloading notation, we use \(\theta\) for vehicle heading and \(\psi(\cdot)\) for the decoder map. At each time step \(t\), the physical state is
\begin{equation}
x_t = [p_x,\; p_y,\; \theta,\; v]^\top \in \mathbb{R}^4,
\label{eq:physical_state}
\end{equation}
where \((p_x,p_y)\) denotes planar position, \(\theta\) is heading, and \(v\) is speed. We also define the state increment
\begin{equation}
\Delta x_t = [\Delta p_x,\; \Delta p_y,\; \Delta \theta,\; \Delta v]^\top \in \mathbb{R}^4.
\label{eq:delta_state}
\end{equation}
Given a history of the past \(K\) increments,
\begin{equation}
H_t = \{\Delta x_{t-K}, \ldots, \Delta x_{t-1}\},
\label{eq:history_state}
\end{equation}
the current state \(x_t\), and the local road-curvature input \(\kappa_t\), the goal is to predict the next increment \(\Delta x_t\) and to roll the model forward over multiple steps in open loop.

We consider both one-step and multi-step rollout training. In the multi-step case, the model is trained by self-feeding its own predicted increments for \(U\) steps, where \(U \in \{1,3,5,10\}\) in our experiments. All models are ultimately evaluated using long-horizon open-loop rollout over a fixed 10-second horizon.

\subsection{Unconstrained Latent Rollout (ULR)}

Our baseline model is the \emph{Unconstrained Latent Rollout} model (ULR). It consists of an encoder, a latent transition model, and a decoder. We first construct the normalized model input
\begin{equation}
o_t = \big[\mathrm{norm}(H_t),\; \mathrm{norm}(x_t),\; \mathrm{norm}(\kappa_t)\big].
\label{eq:norm_obs}
\end{equation}
The latent state and latent increment are computed as
\begin{equation}
z_t = E(o_t), \qquad
\Delta z_t = F([z_t,\kappa_t]),
\label{eq:z_deltaz}
\end{equation}
and the next latent state is updated additively:
\begin{equation}
z_{t+1} = z_t + \Delta z_t.
\label{eq:z_tp1}
\end{equation}
The decoder \(\psi(\cdot)\) maps latent states back to the physical state space:
\begin{equation}
\hat{x}_t = \psi(z_t), \qquad \hat{x}_{t+1} = \psi(z_{t+1}).
\label{eq:decoder_map}
\end{equation}
We define the predicted state increment by decoder differencing:
\begin{equation}
\widehat{\Delta x}_t = \hat{x}_{t+1} - \hat{x}_t.
\label{eq:predicted_state}
\end{equation}
Open-loop rollout then proceeds by integrating this predicted increment from the current rollout state \(x_t^{\mathrm{roll}}\):
\begin{equation}    
x_{t+1}^{\mathrm{roll}} = \mathcal{I}(x_t^{\mathrm{roll}}, \widehat{\Delta x}_t).
\label{eq:rollout_eqn}
\end{equation}
where \(\mathcal{I}(\cdot)\) performs additive integration on position and speed, and wrapped integration on heading:
\begin{equation}
\mathcal{I}(x_t,\Delta x_t)=
\begin{bmatrix}
p_x + \Delta p_x \\
p_y + \Delta p_y \\
\mathrm{wrap}(\theta + \Delta \theta) \\
v + \Delta v
\end{bmatrix}.
\end{equation}
The rollout history is then updated by appending \(\widehat{\Delta x}_t\) and dropping the oldest increment.

ULR is trained using only prediction loss. Over a rollout of length \(U\), we define the increment loss
\begin{equation}
\mathcal{L}_{\Delta x}
=
\frac{1}{U}\sum_{u=0}^{U-1}
\left\|
\mathrm{norm}(\widehat{\Delta x}_{t+u}) -
\mathrm{norm}(\Delta x_{t+u}^{\mathrm{gt}})
\right\|_2^2,
\end{equation}
and the state loss
\begin{equation}
\begin{aligned}
\mathcal{L}_{\mathrm{state}}
&=
\frac{1}{U}\sum_{u=0}^{U-1}
\frac{1}{w_t + w_{t+1}}
\Big(
w_t\,\ell_{\mathrm{state}}(x_{t+u}^{\mathrm{roll}}, x_{t+u}^{\mathrm{gt}})
\\
&\qquad\qquad\qquad\qquad
+
w_{t+1}\,\ell_{\mathrm{state}}(x_{t+u+1}^{\mathrm{roll}}, x_{t+u+1}^{\mathrm{gt}})
\Big).
\end{aligned}
\end{equation}
where \(\ell_{\mathrm{state}}\) is a normalized state error with wrapped heading difference. The overall ULR training objective is
\begin{equation}
\mathcal{L}_{\mathrm{pred}}
=
\mathcal{L}_{\Delta x}
+
\alpha_{\mathrm{state}}\,\mathcal{L}_{\mathrm{state}}.
\end{equation}

\subsection{Latent Pullback Geometry from a Kinematics-Aware Riemannian Metric}

The unconstrained model above optimizes prediction accuracy only. It does not explicitly distinguish between decoded motion that is compatible with local vehicle kinematics and decoded motion that contains excessive lateral drift. To inject such a structure without changing the backbone architecture, we define a motion-aware metric in decoded state space and pull it back to the latent manifold through the decoder Jacobian.

For a decoded state
\begin{equation}
\hat{x}_t = [\hat p_x,\; \hat p_y,\; \hat \theta_t,\; \hat v_t]^\top,
\end{equation}
we define a local feasible tangent subspace in decoded state space. Let
\begin{equation}
g_1(\hat{x}_t)=
\begin{bmatrix}
\cos \hat{\theta}_t\\
\sin \hat{\theta}_t\\
0\\
0
\end{bmatrix},
\qquad
g_2=
\begin{bmatrix}
0\\
0\\
1\\
0
\end{bmatrix},
\qquad
g_3=
\begin{bmatrix}
0\\
0\\
0\\
1
\end{bmatrix}.
\end{equation}
The vector \(g_1\) corresponds to forward translation in the current heading direction, while \(g_2\) and \(g_3\) correspond to heading change and speed change, respectively. Let $P_{\mathrm{feas}}(\hat{x}_t)$ denote the orthogonal projector onto
\begin{equation}
    \mathrm{span}\{g_1(\hat{x}_t), g_2, g_3\}.
\end{equation}
We then define the anisotropic decoded-space metric
\begin{equation}
M_\beta(\hat{x}_t) = I + \beta\big(I - P_{\mathrm{feas}}(\hat{x}_t)\big),
\qquad \beta > 0.
\end{equation}
For \(\beta>0\), \(M_\beta(\hat{x}_t)\) is symmetric positive definite, and therefore defines a kinematics-aware anisotropic Riemannian metric in decoded state space. This metric assigns additional cost to motion outside the feasible subspace, which in the present setting corresponds to the decoded lateral direction.

Let \(J_\psi(z_t)\) denote the Jacobian of the decoder. Pulling the decoded-space metric back through the decoder yields a latent-space metric
\begin{equation}    
G_\beta(z_t) = J_\psi(z_t)^\top M_\beta(\hat{x}_t)\, J_\psi(z_t).
\end{equation}
Thus, the decoder pulls the decoded-space Riemannian metric back to latent space, endowing latent transitions with a state-dependent geometric cost. In this sense, DPR does not merely penalize decoded motion heuristically; it equips latent transitions with a pullback geometry induced by the vehicle-motion metric in decoded space.

\subsection{Euclidean Pullback Regularization and Directional Pullback Regularization}

For a latent update \(\Delta z_t\) at latent state \(z_t\), the decoder induces the decoded tangent
\begin{equation}
u_t = \frac{J_\psi(z_t)\Delta z_t}{\Delta t}.
\end{equation}
We regularize this tangent using the pullback metric. The per-step metric loss is
\begin{equation}
\ell_{\mathrm{metric}}(t)
=
\left(\frac{J_\psi(z_t)\Delta z_t}{\Delta t}\right)^\top
M_\beta(\hat{x}_t)
\left(\frac{J_\psi(z_t)\Delta z_t}{\Delta t}\right),
\end{equation}
or equivalently,
\begin{equation}
\ell_{\mathrm{metric}}(t)
=
\left(\frac{\Delta z_t}{\Delta t}\right)^\top
G_\beta(z_t)
\left(\frac{\Delta z_t}{\Delta t}\right).
\end{equation}
Averaging over the rollout gives
\begin{equation}
\mathcal{L}_{\mathrm{metric}}
=
\frac{1}{U}\sum_{u=0}^{U-1}\ell_{\mathrm{metric}}(t+u).
\end{equation}
The overall training objective is then
\begin{equation}
\mathcal{L}
=
\mathcal{L}_{\mathrm{pred}}
+
\lambda_{\mathrm{reg}}\,\mathcal{L}_{\mathrm{metric}}.
\end{equation}

This formulation yields two regularized variants of the baseline:

\paragraph*{Euclidean Pullback Regularization (EPR)}
When \(\beta = 0\), the decoded-space metric reduces to
\begin{equation}
M_0(\hat{x}_t)=I,
\end{equation}
and therefore
\begin{equation}
G_0(z_t)=J_\psi(z_t)^\top J_\psi(z_t).
\end{equation}
In this case, the regularizer penalizes the overall decoded tangent magnitude without imposing any directional preference. EPR is therefore the Euclidean pullback ablation of the method.

\paragraph*{Directional Pullback Regularization (DPR)}
When \(\beta > 0\), the metric assigns additional cost to the component of decoded motion orthogonal to the feasible subspace. This is the full proposed method. In DPR, the latent manifold is endowed with an anisotropic pullback geometry that penalizes latent transitions according to the local geometry of feasible vehicle motion.

The resulting three-way decomposition is therefore:
\begin{itemize}
    \item \textbf{ULR}: unconstrained latent rollout, with prediction loss only;
    \item \textbf{EPR}: Euclidean pullback regularization, corresponding to \(\beta=0\);
    \item \textbf{DPR}: directional pullback regularization, corresponding to \(\beta>0\).
\end{itemize}
This decomposition is important because it allows us to distinguish gains due to generic smoothing from gains specifically for imposing a directional geometric penalty on decoded lateral motion.

\subsection{Interpretation}

The proposed method is a lightweight geometric bias on latent rollout rather than a hard kinematic model. The encoder, latent dynamics, decoder, and rollout mechanism remain unchanged. What changes is the geometry used to score latent transitions during training. ULR learns by prediction loss alone. EPR penalizes the overall size of the decoded tangent. DPR goes further by shaping the latent manifold with a pullback metric derived from a motion-aware decoded-space metric, thereby making latent directions that decode to lateral motion more expensive. In this sense, DPR encourages motion-consistent rollouts by acting directly on the geometry of latent transitions.

\subsection{Local Lateral-Drift Bound}
This result shows that up to the linearization error of the decoder, reducing the metric loss controls the decoded lateral component of the rollout increment. This local bound is intended to clarify the mechanism of DPR rather than to establish global rollout accuracy or closed-loop stability.

For a rollout segment of length \(H\), define
\begin{equation}
\bar{\mathcal L}_{\mathrm{metric}}^{(H)}
:=
\frac{1}{H}\sum_{t=1}^{H}\ell_{\mathrm{metric}}(t).
\end{equation}

\begin{proposition}
Assume \(J_\psi\) is \(L_J\)-Lipschitz and that
\(\|\Delta z_t\|\le \rho\) for \(t=1,\dots,H\).
Let
\begin{equation}
u_t := \frac{J_\psi(z_t)\Delta z_t}{\Delta t},
\qquad
\tilde u_t := \frac{\psi(z_t+\Delta z_t)-\psi(z_t)}{\Delta t},
\end{equation}
and define the local lateral direction as
$n_t = \begin{bmatrix} -\sin\hat{\theta}_t & \cos\hat{\theta}_t & 0 & 0 \end{bmatrix}^\top$, $\hat x_t = \psi(z_t).$ Then, we have
\begin{equation}
\frac{1}{H}\sum_{t=1}^{H}(n_t^\top \tilde u_t)^2
\le
\frac{2}{\beta}\,\bar{\mathcal L}_{\mathrm{metric}}^{(H)}
+
\frac{L_J^2\rho^4}{2\Delta t^2}.
\end{equation}
Consequently,
\begin{equation}
\sum_{t=1}^{H}
\left|
n_t^\top\bigl(\psi(z_t+\Delta z_t)-\psi(z_t)\bigr)
\right|
\le
H\Delta t
\sqrt{
\frac{2}{\beta}\,\bar{\mathcal L}_{\mathrm{metric}}^{(H)}
+
\frac{L_J^2\rho^4}{2\Delta t^2}
}.
\end{equation}
\end{proposition}

\begin{proof}
Since \(I-P_{\mathrm{feas}}(\hat x_t)=n_tn_t^\top\), the decoded-space metric satisfies $M_\beta(\hat x_t)=I+\beta n_tn_t^\top$, and hence
\begin{equation}
\ell_{\mathrm{metric}}(t)
=
u_t^\top M_\beta(\hat x_t)u_t
=
\|u_t\|^2+\beta(n_t^\top u_t)^2
\ge
\beta(n_t^\top u_t)^2.
\end{equation}
Now write the true decoded increment velocity as
\begin{equation}
\tilde u_t = u_t + r_t,
\qquad
r_t :=
\frac{\psi(z_t+\Delta z_t)-\psi(z_t)-J_\psi(z_t)\Delta z_t}{\Delta t}.
\end{equation}
Because \(J_\psi\) is \(L_J\)-Lipschitz and \(\|\Delta z_t\|\le \rho\),
the first-order remainder obeys
\begin{equation}
\|r_t\|
\le
\frac{L_J\rho^2}{2\Delta t}.
\end{equation}
Therefore
\begin{equation}
(n_t^\top\tilde u_t)^2
\le
2(n_t^\top u_t)^2 + 2\|r_t\|^2
\le
\frac{2}{\beta}\ell_{\mathrm{metric}}(t)
+
\frac{L_J^2\rho^4}{2\Delta t^2}.
\end{equation}
Averaging over \(t\) gives the first claim. Summing the stepwise lateral displacements and applying Cauchy--Schwarz yields the second bound.
\end{proof}

\begin{figure}[t]
  \centering
  \includegraphics[width=\columnwidth]{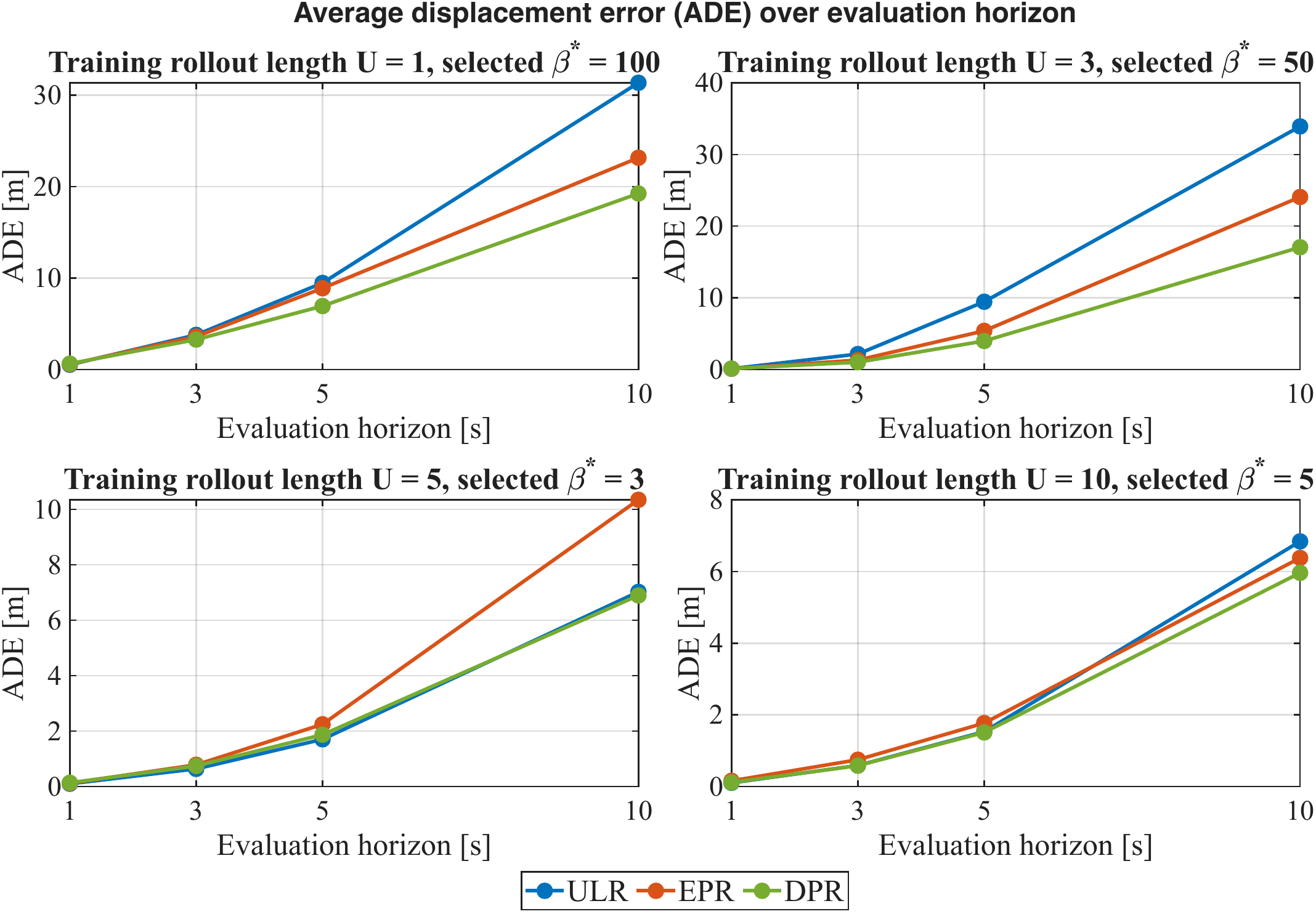}
  \caption{Average displacement error over the evaluation horizon for the three compared models: the unconstrained baseline (ULR), the Euclidean pullback ablation (EPR), and the proposed directional pullback model (DPR). Each subplot corresponds to a different training rollout length \(U \in \{1,3,5,10\}\), and DPR is shown with the validation-selected directional weight \(\beta^*\) for that training regime. Across all training rollout lengths, DPR reduces long-horizon displacement error relative to both ULR and EPR, with the strongest gains appearing for longer self-fed rollout training.}
  \label{fig:adevshorizon}
\end{figure}

\section{Experiments}
\subsection{Experimental Setup}
We evaluate the proposed pullback-regularized model on long-horizon open-loop ego-motion rollout prediction. The physical state is $x_t = [p_x,\; p_y,\; \theta,\; v]^\top$, where \((p_x,p_y)\) is planar position, \(\theta\) is heading, and \(v\) is speed. The model input consists of the previous \(K=4\) state increments, the current state, and the local road curvature. All models use the same encoder--dynamics--decoder backbone and the same latent dimension \(d_z=64\). The backbone uses a lightweight feedforward architecture consisting of a two-layer MLP encoder (\(K\!\cdot\!4+4+1 \rightarrow 256 \rightarrow d_z\)), a two-layer MLP latent dynamics model (\(d_z+1 \rightarrow 256 \rightarrow d_z\)), and a two-layer MLP decoder (\(d_z \rightarrow 128 \rightarrow 4\)), with \(d_z=64\). The only difference between models is the regularization applied during training. We intentionally use a lightweight shared backbone to isolate the effect of the proposed geometric regularizer; evaluating DPR with stronger sequence architectures is an important direction for future work.

All experiments are conducted on ego-motion trajectories derived from the NVIDIA open-source driving dataset \cite{nvidiaavdataset}. We use the released egomotion streams and construct trajectory windows from the physical state and its corresponding increment
\[
\Delta x_t = [\Delta p_x,\; \Delta p_y,\; \Delta \theta,\; \Delta v]^\top.
\]
The raw egomotion timestamps are irregular, so all sequences are resampled to a uniform rate of \(10\) Hz before preprocessing. We then extract fixed-length windows of \(20\) seconds, yielding \(200\) steps per trajectory window, and compute the local road curvature at each step as an additional model input. We use the standard training /validation / test split provided with the data set, consisting of \(126{,}839\) training clips, \(65{,}971\) validation clips, and \(35{,}175\) test clips. All quantitative results reported in this paper are computed on test windows that are disjoint from the training and validation data.
We compare three model variants:
\begin{enumerate}
    \item \textbf{ULR} (\emph{Unconstrained Latent Rollout}): the baseline model trained only with prediction loss;
    \item \textbf{EPR} (\emph{Euclidean Pullback Regularization}): the pullback-regularized ablation obtained by setting \(\beta=0\), which penalizes the overall magnitude of the decoded tangent without introducing directional preference;
    \item \textbf{DPR} (\emph{Directional Pullback Regularization}): the full proposed method, in which \(\beta>0\) assigns additional cost to the lateral component of the decoded tangent.
\end{enumerate}
\begin{figure*}[t]
  \centering
  \includegraphics[width=\textwidth]{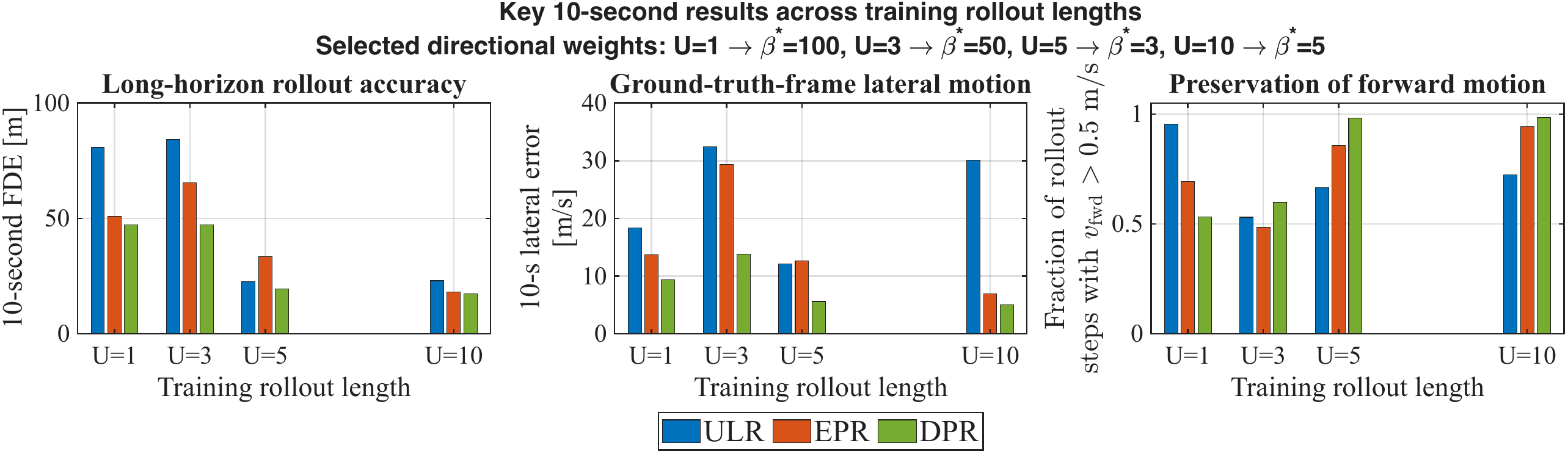}
  \caption{Key 10-second results across training rollout lengths for the unconstrained baseline (ULR), the Euclidean pullback ablation (EPR), and the proposed directional pullback model (DPR). The panels show final displacement error, ground-truth-frame lateral error, and the fraction of rollout steps with forward velocity above \(0.5\,\mathrm{m/s}\). DPR is the only variant that consistently improves 10-second final displacement error across all training rollout lengths, and in the longer-rollout regimes it also maintains or improves the forward-valid fraction while reducing lateral-motion error.}
  \label{fig:key10s}
\end{figure*}
We train separate models with rollout training lengths \(U \in \{1,3,5,10\}\). We chose \(U \in \{1,3,5,10\}\) to span short, intermediate, and long self-fed rollout regimes while keeping the training sweep tractable. Here, \(U\) denotes the number of self-fed training steps used during optimization, not the evaluation horizon. At test time, all models are rolled forward in open loop for \(100\) steps, corresponding to a \(10\)-second horizon with \(\Delta t = 0.1\) s. All test results are reported on the same set of \(2000\) held-out windows.

For each training rollout length, we report one selected value of the directional weight \(\beta\):
\[
\beta^\star_1 = 100,\qquad
\beta^\star_3 = 50,\qquad
\beta^\star_5 = 3,\qquad
\beta^\star_{10} = 5.
\]
Each selected DPR model is compared against both ULR and the EPR ablation.

Our primary evaluation focuses on the quality of open-loop state rollout. To assess both trajectory accuracy and motion consistency, we report the following metrics:
\begin{itemize}
    \item \emph{average displacement error}, defined as the mean Euclidean position error accumulated over the rollout horizon;
    \item \emph{final displacement error}, defined as the Euclidean position error at the final rollout step;
    \item the \(95\)th percentile of the absolute lateral-velocity error, measured in the ground-truth heading frame;
    \item the \(95\)th percentile of the curvature-consistency error, defined as \(|\dot{\theta} - v_{\mathrm{fwd}}\kappa|\);
    \item the median forward velocity over the rollout;
    \item the fraction of rollout steps for which the forward velocity exceeds the minimum forward-motion threshold \(v_{\min}=0.5\) m/s.
\end{itemize}

Lower values are better for displacement error, lateral-motion error, and curvature-consistency error. Higher values are better for median forward velocity and for the fraction of rollout steps with positive forward motion.
\subsection{Main Quantitative Results}
Table~\ref{tab:main_results} summarizes the \(10\)-second open-loop rollout results. Among the regimes tested, the largest relative reduction in 10-second final displacement error appears for the 3-step rollout-trained model. At the same time, the strongest overall control-relevant tradeoff is obtained for longer rollout-trained models, especially at \(U=10\), where DPR combines lower displacement error with improved lateral consistency and forward-motion preservation. The main conclusion is that DPR improves over both baselines across all four training rollout lengths. In particular, for every training rollout length \(U\), the selected DPR model achieves the lowest \(10\)-second final displacement error among the three compared models.

Figure~\ref{fig:adevshorizon} shows the full average displacement error trajectories over the evaluation horizon for ULR, EPR, and DPR across all four training rollout lengths. The one-step training regime \((U=1)\) already shows a large gain in long-horizon rollout error. Relative to ULR, DPR reduces the \(10\)-second mean displacement error from \(31.366\) m to \(19.254\) m and the final displacement error from \(80.646\) m to \(47.172\) m. It also reduces the \(95\)th percentile ground-truth-frame lateral error from \(18.370\) m/s to \(9.395\) m/s. However, this gain comes with a clear forward-motion tradeoff: the median forward velocity drops from \(16.345\) m/s in ULR to \(0.946\) m/s in DPR. Thus, for one-step training, directional pullback regularization substantially improves geometric consistency and final error, but also makes the long-horizon rollout more conservative.

The three-step training regime \((U=3)\) provides a stronger and more balanced result. DPR with \(\beta=50\) reduces the \(10\)-second mean displacement error from \(33.913\) m to \(17.057\) m and the final displacement error from \(84.163\) m to \(47.160\) m. It also reduces the \(95\)th percentile lateral error from \(32.377\) m/s to \(13.748\) m/s and the curvature-consistency error from \(1.868\) rad/s to \(0.902\) rad/s. Unlike the one-step model, the three-step DPR model also improves median forward velocity relative to ULR, from \(1.712\) m/s to \(3.115\) m/s, while increasing the fraction of rollout steps with meaningful forward motion from \(0.530\) to \(0.599\). The five-step and ten-step training regimes yield the strongest overall tradeoffs. For \(U=5\), DPR with \(\beta=3\) improves the \(10\)-second mean displacement error from \(7.022\) m to \(6.894\) m and the final displacement error from \(22.470\) m to \(19.399\) m. More importantly, it dramatically reduces the \(95\)th percentile lateral error from \(12.096\) m/s to \(5.593\) m/s while also improving forward behavior, increasing the median forward velocity from \(6.554\) m/s to \(8.185\) m/s and the forward-valid fraction from \(0.666\) to \(0.980\). For \(U=10\), DPR with \(\beta=5\) gives the strongest overall result among all models tested: the \(10\)-second mean displacement error decreases from \(6.839\) m to \(5.962\) m, the final displacement error decreases from \(22.980\) m to \(17.322\) m, the lateral error drops from \(30.114\) m/s to \(4.971\) m/s, and the median forward velocity increases from \(7.268\) m/s to \(9.189\) m/s. Figure~\ref{fig:key10s} summarizes the key 10-second trends across training rollout lengths, highlighting the tradeoff among final displacement error, lateral consistency, and preservation of forward motion.

Overall, these results show that the pullback metric is effective not only in reducing displacement error, but also in shaping the rollout toward lower lateral drift and better kinematic consistency. The most favorable balance between geometric correction and preservation of forward motion is obtained for the longer self-fed training rollouts \(U=5\) and \(U=10\).
\begin{table*}[t]
\centering
\caption{Ten-second open-loop rollout results for ULR (Unconstrained Latent Rollout), EPR (Euclidean Pullback Regularization), and DPR (Directional Pullback Regularization).}
\label{tab:main_results}
\scriptsize
\setlength{\tabcolsep}{4pt}
\renewcommand{\arraystretch}{1.08}
\begin{tabular}{c l c c c c c c c}
\hline
Train rollout & Model & $\beta$ &
\begin{tabular}[c]{@{}c@{}}10 s mean\\displacement\\error (m)\end{tabular} &
\begin{tabular}[c]{@{}c@{}}10 s final\\displacement\\error (m)\end{tabular} &
\begin{tabular}[c]{@{}c@{}}95th lateral\\error\\(m/s)\end{tabular} &
\begin{tabular}[c]{@{}c@{}}95th curvature\\error\\(rad/s)\end{tabular} &
\begin{tabular}[c]{@{}c@{}}Median\\forward\\velocity (m/s)\end{tabular} &
\begin{tabular}[c]{@{}c@{}}Forward-valid\\fraction\end{tabular} \\
\hline

1  & ULR & --  & 31.366 & 80.646 & 18.370 & 0.683 & 16.345 & 0.953 \\
1  & EPR & 0   & 23.165 & 50.927 & 13.689 & 0.687 & 16.983 & 0.692 \\
1  & \textbf{DPR} & 100 & 19.254 & 47.172 &  9.395 & 0.773 &  0.946 & 0.531 \\
\hline

3  & ULR & --  & 33.913 & 84.163 & 32.377 & 1.868 &  1.712 & 0.530 \\
3  & EPR & 0   & 24.067 & 65.300 & 29.309 & 1.716 & -0.623 & 0.483 \\
3  & \textbf{DPR} & 50  & 17.057 & 47.160 & 13.748 & 0.902 &  3.115 & 0.599 \\
\hline

5  & ULR & --  &  7.022 & 22.470 & 12.096 & 0.257 &  6.554 & 0.666 \\
5  & EPR & 0   & 10.344 & 33.475 & 12.646 & 0.248 &  6.951 & 0.856 \\
5  & \textbf{DPR} & 3   &  6.894 & 19.399 &  5.593 & 0.127 &  8.185 & 0.980 \\
\hline

10 & ULR & --  &  6.839 & 22.980 & 30.114 & 0.325 &  7.268 & 0.722 \\
10 & EPR & 0   &  6.378 & 18.063 &  6.882 & 0.087 &  5.546 & 0.942 \\
10 & \textbf{DPR} & 5   &  5.962 & 17.322 &  4.971 & 0.060 &  9.189 & 0.984 \\
\hline
\end{tabular}
\end{table*}
\subsection{Ablation: Euclidean Versus Directional Pullback Regularization}

A central question is whether the observed gain is due merely to adding any regularization term, or whether the directional lateral penalty is the crucial ingredient. The comparison between EPR and DPR answers this question directly.

EPR already changes the behavior of the latent rollout model because it penalizes the overall magnitude of the decoded tangent. In some cases this is beneficial. For example, when the training rollout length is \(10\), EPR reduces the final displacement error from \(22.980\) m in ULR to \(18.063\) m. However, Euclidean pullback regularization alone is not sufficient. It does not consistently produce the best long-horizon rollout and can even hurt performance. For training rollout length \(5\), EPR increases the \(10\)-second final displacement error from \(22.470\) m to \(33.475\) m, even though DPR reduces it to \(19.399\) m.

The directional term consistently improves on the Euclidean ablation. For training rollout length \(1\), the final displacement error decreases from \(50.927\) m with EPR to \(47.172\) m with DPR, and the \(95\)th percentile lateral error decreases from \(13.689\) m/s to \(9.395\) m/s. For training rollout length \(3\), the final displacement error decreases from \(65.300\) m to \(47.160\) m, and the lateral error decreases from \(29.309\) m/s to \(13.748\) m/s. Notably, for \(U=3\), EPR yields a negative median forward velocity (\(-0.623\,\mathrm{m/s}\)), indicating that backward motion occurs in more than half of the evaluated rollout steps; this failure mode reinforces that isotropic shrinkage alone can distort motion realism and that directional structure is necessary. For training rollout length \(5\), DPR reverses the failure mode of EPR and improves both displacement error and forward behavior. For training rollout length \(10\), DPR again improves over EPR, reducing final displacement error from \(18.063\) m to \(17.322\) m and increasing median forward velocity from \(5.546\) m/s to \(9.189\) m/s.

This ablation is important because it shows that the gains do not come from generic damping alone. The extra penalty on the lateral decoded tangent is the component that makes the regularizer effective. In other words, the model benefits not just from shrinking decoded motion, but from shrinking it in a geometrically meaningful direction.

\subsection{Effect of Training Rollout Length}

The duration of the training rollout changes how the pullback regularizer interacts with long-horizon prediction. One-step training can already produce large reductions in long-horizon displacement error, but the selected DPR model becomes overly conservative in forward motion. As the training rollout length increases, the selected pullback model becomes better at balancing displacement accuracy, lateral consistency, and preservation of forward motion.

This trend is especially clear when comparing the selected DPR models across \(U=1,3,5,10\). The one-step model achieves a large reduction in final displacement error but nearly stalls by the end of the rollout. The three-step model offers a better balance, reducing both long-horizon error and lateral drift while preserving more forward motion. The five-step and ten-step models are stronger still: they improve long-horizon position error, reduce lateral drift substantially, and maintain high fractions of forward-valid rollout steps. Among the models tested, the ten-step DPR model provides the strongest overall \(10\)-second result.

These observations suggest that pullback regularization is most effective when the model is trained under sufficiently long self-fed rollout, so that the geometric penalty acts on the same type of error accumulation that appears at test time. Short rollout training can still benefit from the regularizer, but is more prone to the tradeoff between reduced lateral drift and reduced forward aggressiveness.

\subsection{Summary of Experimental Findings}

The experiments lead to three main findings. First, DPR improves long-horizon rollout over the unconstrained latent baseline ULR. Second, the gain is not explained by Euclidean pullback regularization alone; the directional lateral penalty is necessary for the strongest results. Third, the quality of the final tradeoff depends on the training rollout length: short-rollout training can over-regularize motion, whereas longer self-fed rollout training allows the proposed method to improve displacement accuracy, geometric consistency, and preservation of forward motion simultaneously. Among the regimes tested, \(U=3\) shows the strongest relative improvement in 10-second final displacement error over ULR, while \(U=10\) provides the strongest overall tradeoff among displacement accuracy, lateral consistency, and forward-motion preservation.

\section{Limitations}
Several aspects of this study remain open. First, the directional weight \(\beta\) is selected by validation and its preferred value varies with the training rollout length, indicating that the strength of the geometric bias interacts with the amount of self-fed rollout. Second, stronger lateral-motion control can trade off against preservation of forward motion, particularly for shorter rollout-trained models. Third, the pullback metric acts locally on decoded tangents and does not by itself guarantee global long-horizon accuracy or closed-loop stability. Finally, the experiments are limited to ego-motion rollout; multi-agent interaction, action-conditioned planning, and closed-loop evaluation are left for future work.

\section{Conclusion and Future Work}

This paper introduced \emph{Directional Pullback Regularization} (DPR) for long-horizon latent ego-motion rollout in autonomous driving. By defining an anisotropic metric in decoded state space and pulling it back through the decoder to the latent space, the method assigns a motion-aware geometry to latent transitions without changing the underlying encoder--dynamics--decoder backbone. The resulting experimental comparison among \emph{Unconstrained Latent Rollout} (ULR), \emph{Euclidean Pullback Regularization} (EPR), and DPR shows that the directional pullback term improves long-horizon rollout beyond both the unconstrained baseline and the Euclidean pullback ablation. In particular, the results show that the gains are not explained by generic smoothing alone but by explicitly penalizing latent transitions that decode to lateral vehicle motion. Empirically, the largest relative reduction in 10-second final displacement error appears in the \(U=3\) regime, while the \(U=10\) regime provides the strongest overall control-relevant balance of error reduction, lateral-motion suppression, and forward-motion preservation.

Several directions remain for future work. A natural next step is to develop more principled strategies for selecting or adapting the directional weight \(\beta\). It is also important to extend the present formulation beyond ego-motion rollout to richer multi-agent and action-conditioned settings. Finally, studying how the proposed latent-space geometry affects downstream planning and closed-loop control remains an important direction for connecting improved rollout consistency to decision-making performance.

\end{document}